\documentclass[sigconf]{acmart}
\usepackage{natbib}
\usepackage{booktabs} % For formal tables
\graphicspath{{figure/}}

{
\theoremstyle{plain}
\newtheorem{problem}{Problem}

\newtheorem{example}{Example}

}

\setcopyright{rightsretained}
\begin{document}
\title{Learning Correlation Space for Time Series}
%\titlenote{Produces the permission block, and
 % copyright information}
%\subtitle{Extended Abstract}
%\subtitlenote{The full version of the author's guide is available as
 %\texttt{acmart.pdf} document}

\author{Han Qiu}
 \affiliation{  
  \institution{Massachusetts Institute of Technology}
   \streetaddress{77 Massachusetts Ave}
   \city{Cambridge, MA, USA}
 }
 \email{hanqiu@mit.edu}

 \author{Hoang Thanh Lam}
 \affiliation{
   \institution{IBM Research}
   \streetaddress{IBM Technology Campus}
   \city{Dublin, Ireland}
  }
 \email{t.l.hoang@ie.ibm.com}

 \author{Francesco Fusco}
 \affiliation{%
   \institution{IBM Research}
   \streetaddress{IBM Technology Campus}
   \city{Dublin, Ireland}
  }
 \email{francfus@ie.ibm.com}

 \author{Mathieu Sinn}
 \affiliation{%
   \institution{IBM Research}
   \streetaddress{IBM Technology Campus}
   \city{Dublin, Ireland}
  }
 \email{mathsinn@ie.ibm.com}

%The default list of authors is too long for headers.
 %\renewcommand{\shortauthors}{Han Qiu \textit{et al.}}

\begin{abstract}

We propose an approximation algorithm for efficient correlation search in time series data. In our method, we use Fourier transform and neural network to embed time series into a low-dimensional Euclidean space. The given space is learned such that time series correlation can be effectively approximated from Euclidean distance between corresponding embedded vectors. Therefore, search for correlated time series can be done using an index in the embedding space for efficient nearest neighbor search. Our theoretical analysis illustrates that our method's accuracy can be guaranteed under certain regularity conditions. We further conduct experiments on real-world datasets and the results show that our method indeed outperforms the baseline solution. In particular, for approximation of correlation, our method reduces the approximation loss by a half in most test cases compared to the baseline solution. For top-$k$ highest correlation search, our method improves the precision from 5\% to 20\% while the query time is similar to the baseline approach query time.

\end{abstract}



\maketitle

\section{Introduction}
Given a massive number of time series, building a compact index of time series for efficient correlated time series search queries is an important research problem \cite{Agrawal_1993,Chan_1999}. The classic solutions \cite{Agrawal_1993,Zhu_2002,Mueen_2010} in the literature use Discrete Fourier Transformation (DFT) to transform time series into the frequency domain and approximate the correlation using only the first few coefficients of the frequency vectors. Indeed, people have shown that using only the first 5 coefficients of the DFT is enough to approximate the correlation among stock indices with high accuracy \cite{Zhu_2002}.

Approximation of a time series using the first few coefficients of its Fourier transformation can be considered as a dimension-reduction method that maps long time series into a lower dimensional space with minimal loss of information.
An advantage of such dimension-reduction approaches is that they are unsupervised methods and are independent from use-cases; therefore, they can be used for many types of search queries simultaneously. However, they might not be ideal for applications on particular use-cases where the index is designed just for a specific type of query.

In practice, depending on situations we may serve different types of search queries such as top-$k$ highest correlation search \cite{Dallachiesa_2014}, threshold-based (range) correlation search \cite{Agrawal_1993,Mueen_2010} or even simple approximation of the correlation between any pair of time series.
Different objective function might be needed to optimize performance of different query types. For instance, for the first two query types it is important to preserve the order of the correlation in the approximation, while for the third one it is more important to minimize the approximation error.

\begin{figure}
    \begin{center}
        \includegraphics[width = 1.0\columnwidth]{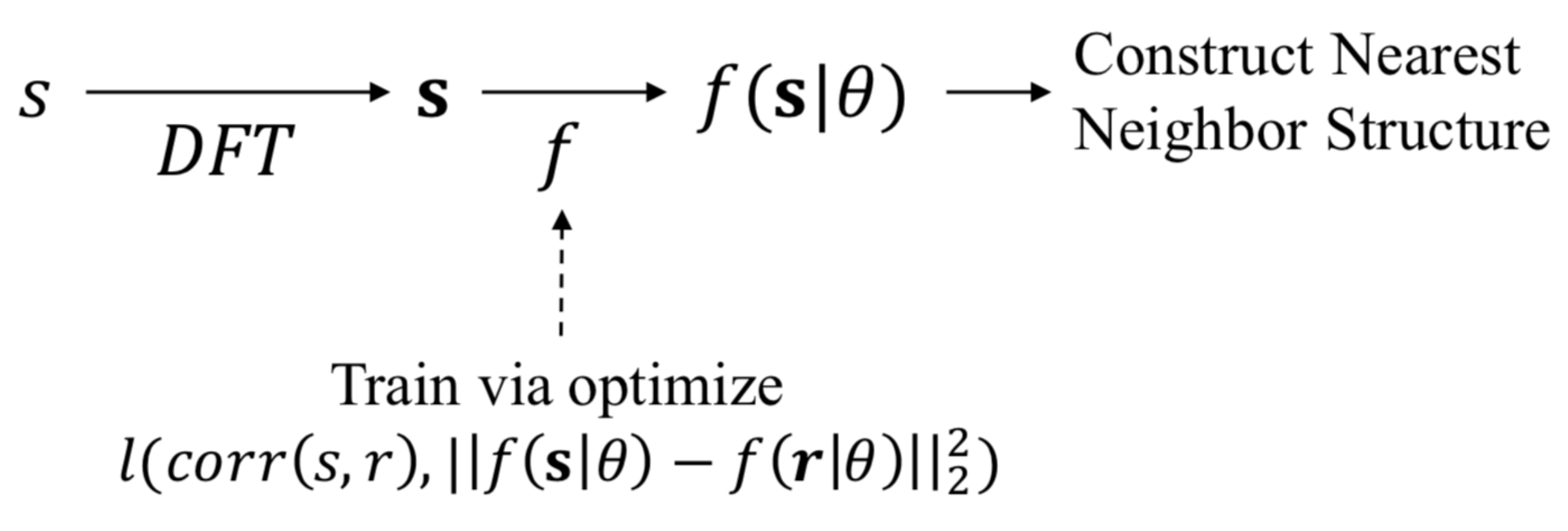}
    \caption{Solution Framework}
    \label{fig:framework}
    \end{center}
\end{figure}

To overcome such problems, in this paper we propose a general framework to optimize query performance for specific combination of datasets and query types.
The key idea behind this framework is sketched in Figure \ref{fig:framework}. Denote $s$ as a time series in the set $S$ and $\mathbf{s} = DFT(s)$ as the discrete Fourier transformation of $s$.
We use a neural network $f(\cdot|\theta)$ with parameters $\theta$ to map $\mathbf{s}$ into a low-dimensional space vector $f(\mathbf{s}|\theta)$. This neural network tries to approximate the correlation between a pair of time series $s$ and $r$ using the Euclidean distance $\|f(\mathbf{s}|\theta)-f(\mathbf{r}|\theta)\|_2$ on the embedding space $f(S)$. We can then utilize classic nearest neighbor search algorithms on this embedding space to index time series and improve query efficiency.

We also notice that specific loss function is needed in training the network to optimize performance of specific type of query. In experiments, we will see that using the right loss function for a specific type of query is crucial in achieveving good results. Therefore, in this paper we also propose relevant loss functions for all aforementioned search query types.

In summary, the contributions of this work are listed as follows:
\begin{itemize}
\item we propose a framework to approximate time series correlation using a time series embedding method;
\item we do theoretical analysis to show non-optimality of the baseline approach from which it motivates this work. We proposed appropriate loss function for each type of query and provide theoretical guarantee for these approaches;
\item we conduct experiments on real-world datasets to show the effectiveness of our approach compared to the classic solutions \cite{Agrawal_1993}: our approach reduces the approximation error at least by a half in the experiments, and improves the precision for top-$k$ correlation search from 5\% to 20\%;
\item we open  the source code for reproducibility of our research\footnote{Omitted for blind review}.
\end{itemize}

The rest of the paper is organized as follows. We discuss background of the paper in section \ref{sec:background}. Three important problems are formulated in section \ref{sec:problem}. In section \ref{sec:theory} we show the non-optimality of the baseline approach from which it motivates this work. We also show theoretical guarantee of  proposed loss functions for each problem formulated in section \ref{sec:problem}. Sections \ref{sec:design} and \ref{sec:train} show how to build  and train the embedding neural networks. Experimental results are discussed in section \ref{sec:experiments} and section \ref{sec:final} concludes the paper.
\section{Background}
\label{sec:background}
 \begin{table}
  \begin{tabular}{ | c | c |}
    \hline
    \textbf{Notations} & \textbf{Meanings}  \\ \hline
    $s,r,u$ & time series \\ \hline
    $\mathbf{s}$ &  DFT of $s$   \\ \hline
     $\bar{s}$ &  Mean value of $s$   \\ \hline
     $\sigma_s$ &  Standard deviation of $s$   \\ \hline
     $\hat{s}$ &  Normalized series of $s$   \\ \hline
     $N^+$ & the set of all positive integers\\ \hline
    $M$ &  time series length   \\ \hline
    $m$ &  embedding size   \\ \hline
    $k$ &  selection set size  \\ \hline
    $\eta$ & selection threshold \\ \hline
    $F$ & top-$k$ selection set \\ \hline
    $\theta$ &  network parameters \\ \hline
    $f(s|\theta)$ &  neural network with parameter $\theta$   \\ \hline
    $\rho$ & query precision \\ \hline
    $\delta$ & query approximation gap \\ \hline
    $\|s-r\|_2$ &  Euclidean distance between $s$ and $r$  \\ \hline
    $d_m(\mathbf{s},\mathbf{r})$ &  Euclidean distance between $\mathbf{s}$ and $\mathbf{r}$   \\
    & considering only the first $m$ elements  \\ \hline
  \end{tabular}
  \caption{Notations} \label{tab:note}
\end{table}
Let $s=[s_1,s_2,\cdots,s_M],r=[r_1,r_2,\cdots,r_M]$ be two time series. The Pearson correlation between $r$ and $s$ is defined as:
\begin{equation}
    corr(s,r) = \frac{\frac{1}{M}\sum_{j=1}^M s_jr_j-\bar{s}\bar{r}}{\sigma_s\sigma_r},
\end{equation}
where $\bar{s} = \sum_{j=1}^M s_j / M$ and $\bar{r} = \sum_{j=1}^M r_j / M$ are the mean values of $s$ and $r$, while $\sigma_s = \sqrt{\sum_{j=1}^M(s_j - \bar{s})^2/M}$ and $\sigma_r = \sqrt{\sum_{j=1}^M(r_j - \bar{r})^2/M}$ are the standard deviations.
If we further define the $l_2$-normalized vector of the time series $s$, $\hat{s} = [\hat{s}_1,\hat{s}_2,\cdots,\hat{s}_M]$, as:
\begin{equation}
    \hat{s}_j = \frac{1}{\sqrt{M}}\frac{s_j-\bar{s}}{\sigma_s} = \frac{s_j-\bar{s}}{\sqrt{\sum_{j=1}^M(s_j - \bar{s})^2}},
\end{equation}

\noindent we can reduce $corr(s,r)$ as $\hat{s}^T\hat{r} = 1 - \|\hat{s} - \hat{r}\|_2^2 / 2$. That is, the correlation between two time series can be expressed as a simple function on the corresponding $l_2$-normalized time series. Therefore, in following discussion, we always consider $l_2$-normalized version $\hat{s}$ of a time series $s$.

The scaled discrete Fourier transformation of the time series $s$, $\mathbf{s}=[\mathbf{s}_1,\mathbf{s}_2,\cdots,\mathbf{s}_M]$, is defined as:
\begin{equation}
    \mathbf{s}_j = \frac{1}{\sqrt{M}}\sum_{l=0}^{M-1} s_le^{\frac{-2jl\pi i}{M}}
\end{equation}
where $i^2 = -1$. We use ``scaled'' here for the factor $\frac{1}{\sqrt{M}}$, which made $\mathbf{s}$ having the same scale as $s$. By the Parseval's theorem \cite{Rudin_1964,Agrawal_1993}, we have $s^Tr = \mathbf{s}^T\mathbf{r}^*$ and $\|s\|_2 = \|\mathbf{s}\|_2$. We can then recover the following property in the literature \cite{Zhu_2002,Mueen_2010}:
\begin{itemize}
\item if $corr(s, r) = s^Tr > 1-\varepsilon^2/2$ then $d_m(\mathbf{s}, \mathbf{r}) < \varepsilon$, where $d_m$ is the Euclidean distance mapped on the first $m$ dimensions with $m < \frac{M}{2}$
\begin{equation}
    d_m^2([\mathbf{s}_1,\mathbf{s}_2,\cdots,\mathbf{s}_M],[\mathbf{r}_1,\mathbf{r}_2,\cdots,\mathbf{r}_M]) = \sum_{j=1}^m\|\mathbf{s}_j - \mathbf{r}_j\|_2^2.
\end{equation}
\end{itemize}

Based on this property, authors in \cite{Agrawal_1993} develop a method to search for highest correlated time series in a database utilizing $d_m^2(\mathbf{s}, \mathbf{r}) \approx 2 - 2\cdot corr(s,r)$. This method is considered as the baseline solution in this paper.

\section{Problem formulation}
\label{sec:problem}
Assume we are given a dataset of regular time series $S\in R^M$, where $M\in N^+$ is a large number. By ``regular'' we mean that the time series are sampled from the same period and with the same frequency (and therefore, having same length). Some correlation search problems can be formulated as follows:

\begin{problem}[Top-$k$ Correlation Search]\label{prob:top-k}
Find a method $F$ that for any given time series $s\in S$ and count $k\in N^+$, $F(s,k|S)$ provides a subset of $S$ that includes the $k$ highest correlated time series with respect to $s$. More formally,
\begin{equation}\label{F}
    \begin{aligned}
    F(s,k|S) & = \arg\max_{\{ S'|S'\subset S,|S'|=k \}}\sum_{r\in S'} corr(s,r) \\
    & = \arg\min_{\{ S'|S'\subset S,|S'|=k \}}\sum_{r\in S'} \|s-r\|_2^2
    \end{aligned}
\end{equation}
\end{problem}

\begin{problem}[Threshold Correlation Search]\label{prob:thres}
Find a method $G$ that for any given time series $s\in S$ and threshold value $\eta\in (0,1)$, $G(s,\eta|S)$ provides a subset of $S$ that consists of all time series $r$ with $corr(s,r) \geq \eta$. More formally,
\begin{equation}
    G(s,\delta|S) = \{r\in S| corr(s,r) \geq \eta\}
\end{equation}
\end{problem}

\begin{problem}[Correlation Approximation]\label{prob:approximation}
Find an embedding function $f: R^M \rightarrow R^m$ that minimizes the expected approximation error in correlations
\begin{equation}
    f(\cdot|S) = \arg\min_{f} E_{s,r\in S}|\|f(s)-f(r)\|_2^2 - 2(1-corr(s,r))|
\end{equation}
\end{problem}

Solutions of these problems are related with each other. First, Problem \ref{prob:top-k} and Problem \ref{prob:thres} are almost equivalent. For instance, if the set $\{corr(s,r)|r\in S\}$ contains no identical elements, and if we let $\eta(s,k|S)$ to be the $k$th largest element in this set, we will have $G(s,\eta(s,k|S)|S) = F(s,k|S)$. Therefore, we only need to discuss one of them; in this paper we will focus on Problem \ref{prob:top-k}.

Second, solutions of Problem \ref{prob:approximation} can lead to good approximation for Problem \ref{prob:top-k}. If there is an $f$ such that the objective value of Problem \ref{prob:approximation} is close to 0, or
\begin{equation}\nonumber
    \|f(s) - f(r)\|_2^2 \approx 2(1 - corr(s,r)),
\end{equation}

\noindent we can consider the following approximation of Problem  \ref{prob:top-k}
\begin{equation}\label{F_hat}
    \hat{F}(s,k|f,S) = \arg\min_{\{ S'|S'\subset S,|S'|=k \}}\sum_{r\in S'} \|f(s) - f(r)\|_2^2
\end{equation}

\noindent combined with top-$k$ nearest neighbor search structures on the embedding space $f(S)$. If the dimension $m$ of the embedded space is small, we can use $k$-d tree \cite{Bentley_1975}. With a balanced $k$-d tree, the complexity of searching the top $k$ nearest neighbors among $n$ candidates reduces to $O(k\log n)$ \cite{Friedman_1977}.

All problems defined above have several important applications. In time series forecasting, one can search for strongly correlated time series in historical database and then use them to predict the target time series. In data storage, when the whole time series dataset is very large and does not fit available memory, one can compress the time series significantly using the embedding with small trade-off for loss due to correlation approximation.

\section{Theoretical analysis}
\label{sec:theory}
This section discusses the theoretical analysis to support our methods. We first show that DFT is not an optimal solution for Problem \ref{prob:top-k} and Problem \ref{prob:approximation} in general. This result motivates our work as there is room for improvement over the baseline if data has specific structure that our learning algorithm can leverage. Next we show that in order to solve Problem \ref{prob:top-k}, it suffices to approximate the $l_2$ norm on $S$ with $l_2$ norm on $f(S)$. We then propose an appropriate loss function to solve Problem \ref{prob:top-k} effectively.

\subsection{Fourier transform's non-optimality}
In the literature \cite{Zhu_2002}, it has been shown that using only the first 5 DFT coefficients, it is enough to approximate solutions for Problem \ref{prob:top-k} on stocks data with high accuracy. However, there is still much room for improvement over the methods based on DFT as shown in the experimental section. This is because DFT approximations blindly assumes that most of the energy (information) in the time series is in the low-frequency components. In the two following examples, we illustrate that the naive DFT method might neither accurately approximate the correlation function nor extract the most important information explaining dataset variability, if the information in the high-frequency components of the time series is non-negligible. Use-case-specific dimension reduction methods such as neural network can avoid such bias and therefore are more preferable.

\begin{example}\label{example:approx}
    Consider a set $S$ of real-valued time series $s$, where the corresponding DFT vector $\mathbf{s}$ of each $s$ has following property
    \begin{equation}\nonumber
        \mathbf{s}_i = \mathbf{s}_{M/4+i}, \ \forall i=1,\cdots,M/4.
    \end{equation}

    Recall that for DFT vector $\mathbf{s}$ of real-valued time series $s$, we have conjugacy $\mathbf{s}_i = \mathbf{s}^*_{M+1-i}$ for $i=1,\cdots,M$. That implies $4d_m^2(\mathbf{s}, \mathbf{r}) = \|\mathbf{s} - \mathbf{r}\|_2^2 = 2 - 2corr(s,r)$ for $m=M/4$.
    The optimal embedding with size $m$ for the correlation function approximation should therefore be $f(\mathbf{s}) = [2\mathbf{s}_1,2\mathbf{s}_2,\cdots,2\mathbf{s}_m]$ which differs from the embedding using DFT a constant factor of $2$.
\end{example}

\begin{example}\label{example:order}
    Consider a set $S$ of similar time series $s$, where the corresponding DFT vector $\mathbf{s}$ of each $s$ is generated as follows:
    \begin{equation}
        \begin{aligned}
            \mathbf{s}_i & = \mathbf{\mu}_{i}^0 + z_i\mathbf{\sigma}_{i}^0, \ \forall i=1,\cdots,M \\
            z_i & \sim
            \begin{cases}
                N(0,\varepsilon) & \forall i=1,\cdots,m\\
                N(0,1) & \forall i=m+1,\cdots,M\\
            \end{cases}
        \end{aligned}
    \end{equation}
    \noindent where $\varepsilon \ll 1$, $\mathbf{\sigma}_{i}^0\neq 0$, and $z_i$ are sampled independently.
    It is obvious that for Problem \ref{prob:top-k}, DFT method has almost same precision as random guess.
\end{example}

\subsection{Approximating top-$k$ correlated series}
In this subsection, we analyze the theoretical performance guarantee of approximation \eqref{F_hat}. Using notation $F(s,k|S)$ and $\hat{F}(s,k|f,S)$, Problem \ref{prob:top-k} is equivalent to maximize the precision $\rho$ on $f$ given $s$ and $k$
\begin{equation}\label{opt_accu}
    \rho(s,k,f|S) = \frac{1}{k}|\hat{F}(s,k|f,S) \bigcap F(s,k|S)|,
\end{equation}

\noindent or to minimize the approximation gap $\delta$ on $f$
\begin{equation}\label{opt_dist}
    \delta(s,k,f|S) = \frac{1}{k}[\sum_{r\in \hat{F}(s,k|f,S)} \|s-r\|_2^2 - \sum_{r\in F(s,k|S)}\|s-r\|_2^2] .
\end{equation}

If the mapping $f$ well preserves the metric, that is, the gap between $\|s-r\|_2^2$ and $\|f(s)-f(r)\|_2^2$ is uniformly small, we can prove that the approximation gap $\delta$ is also quite small. Formally, we have:
\begin{theorem}\label{theorem:top-k}
	If for a function $f:R^M\to R^m$, we have
    \begin{equation}\nonumber
        |\|f(s)-f(r)\|_2^2 - \|s-r\|_2^2| \leq \varepsilon, \ \forall s,r\in S,
    \end{equation}
    \noindent where $\varepsilon\in R^+$, then we also have
    \begin{equation}\nonumber
        \delta(s,k,f|S) \leq 2\varepsilon, \ \forall s\in S, k\in N^+.
    \end{equation}
\end{theorem}
\begin{proof}
    In fact, we have
    \begin{equation}
    \begin{aligned}
        & \sum_{r\in \hat{F}(s,k|f,S)} \|s-r\|_2^2 \\
        \leq & \sum_{r\in \hat{F}(s,k|f,S)} \|f(s)-f(r)\|_2^2 + k\varepsilon \\
        \leq & \sum_{r\in F(s,k|S)} \|f(s)-f(r)\|_2^2 + k\varepsilon \\
        \leq & \sum_{r\in F(s,k|S)} \|s-r\|_2^2 + 2k\varepsilon
    \end{aligned}
    \end{equation}
    \noindent holds for every $s\in S$ and $k\in N^+$. The second inequality above holds since $\hat{F}(s,k|f,S)$ solves the corresponding minimization problem \eqref{F_hat}. The conclusion then follows trivially.
\end{proof}

Theorem \ref{theorem:top-k} provides us a sanity check that approximation \eqref{F_hat} becomes more accurate (in terms of $\delta$) as function approximation $f$ gets better. We can therefore design the following loss function on parametric mapping $f(\cdot|\theta)$ for optimization
\begin{equation}\label{l_delta}
	L_{\delta}(\theta) = E_{s,r\in S}|\|f(s|\theta)-f(r|\theta)\|_2^2-\|s-r\|_2^2|.
\end{equation}

% Notice that when $\alpha=1$, optimizing \eqref{l_delta} leads to a solution of Problem \ref{prob:approximation}. In experiments, however, this solution does not perform as good as solution from other $\alpha$, in terms of $\rho$ and $\delta$.
% This might be explained by the following example: if $S=B(s,\varepsilon)$ is a small local neighborhood of a time series $s$, with $\alpha=1$ the metric in the embedding space $\|f(s|\theta)-f(r|\theta)\|_2^2$ will be mostly explained by the approximation error of $f$ rather than the metric in the original space $\|s-r\|_2^2$.
% In this case, by making $\alpha$ smaller, we can scale up the importance of the metric in the original space and therefore improve $\rho$ and $\delta$. Such phenomenon informs us that we can consider more general distance metric $d(\cdot,\cdot)$ on $R^m\times R^m$ to improve embedding efficiency:

Theorem \ref{theorem:top-k} can also be extended for more general distance $d(\cdot,\cdot)$ on $R^m\times R^m$.

\begin{theorem}\label{theorem:top-k-gen}
    Consider distance $d(\cdot,\cdot)$ on $R^m\times R^m$ and corresponding set approximation
    \begin{equation}\label{F_hat_gen}
        \hat{F}_d(s,k|f,S) = \arg\min_{\{ S'|S'\subset S,|S'|=k \}}\sum_{r\in S'} d(f(s),f(r)).
    \end{equation}
	If for a function $f:R^M\to R^m$, we have
    \begin{equation}\nonumber
        |d(f(s),f(r)) - \|s-r\|_2^2| \leq \varepsilon, \ \forall s,r\in S,
    \end{equation}
    \noindent where $\varepsilon\in R^+$, then we also have
    \begin{equation}\nonumber
        \delta_d(s,k,f|S) \leq 2\varepsilon, \ \forall s\in S, k\in N^+,
    \end{equation}
    \noindent where $\delta_d$ is defined as
    \begin{equation}
        \delta_d(s,k,f|S) = \frac{1}{k}[\sum_{r\in \hat{F}_d(s,k|f,S)} \|s-r\|_2^2 - \sum_{r\in F(s,k|S)}\|s-r\|_2^2] .
    \end{equation}
\end{theorem}
\begin{proof}
    This theorem can be proved by proceeding with the same argument as in the proof of Theorem \ref{theorem:top-k}.
\end{proof}

Although we have confidence that a good metric-approximating $f$ can lead to low approximation gap $\delta$, there is no much to say about the precision $\rho$.
For instance, assume we are given a group of time series $S$ and a target series $s$, and for each $r\in S$ the correlation $corr(s,r)$ is close to others. Or, we can assume $|corr(s,r) - c_0|<\varepsilon$ where $c_0,\varepsilon$ are fixed values.
In this case, an $f$ with low approximation error $|\|f(s)-f(r)\|_2^2 - \|s-r\|_2^2|$, and therefore low gap $\delta$, might still confuse the relative order of the correlation $corr(s,r)$ and distance $\|s-r\|_2^2|$. This will result in a low precision $\rho$.
To improve practical approximation performance on $\rho$ for Problem \ref{prob:top-k}, we propose a loss function that tries to approximate correlation order

\begin{equation}\label{l_rho}
    \begin{aligned}
	       L_{\rho}(\theta) = E_{s,r,u\in S}| & (\|f(r)-f(s)\|_2^2 - \|f(r)-f(u)\|_2^2) \\
           & - (\|r-s\|_2^2 - \|r-u\|_2^2)|.
    \end{aligned}
\end{equation}

The following corollary from Theorem \ref{theorem:top-k} shows that we can also obtain performance guarantee on $\delta$ for network optimizing $L_{\rho}$.
\begin{corollary}\label{cor:top-k}
	If for a function $f:R^M\to R^m$, we have
    \begin{equation}\nonumber
        \begin{aligned}
            |(\|f(r)-f(s)\|_2^2 - & \|f(r)-f(u)\|_2^2) \\
             - & (\|r-s\|_2^2 - \|r-u\|_2^2)| \leq \varepsilon, \ \forall r, s, u\in S,
        \end{aligned}
    \end{equation}
    \noindent where $\varepsilon\in R^+$, then we also have
    \begin{equation}\nonumber
        \delta(s,k,f|S) \leq 2\varepsilon, \ \forall s\in S, k\in N^+.
    \end{equation}
\end{corollary}
\begin{proof}
This reduce to Theorem \ref{theorem:top-k} if we set $r=u$.
\end{proof}

At the end of this section, we notice that the uniform property on $\epsilon_{s,r} = \|s-r\|_2^2 - d(f(s),f(r))$ might be too restrictive, while in general neural network can only attain some probabilistic bounds. Since we have
\begin{equation}\label{delta_form}
    \begin{aligned}
            \delta_d(s,k,f|S) & \leq \frac{1}{k}[\sum_{r\in \hat{F}_d(s,k|f,S)} \epsilon_{s,r} - \sum_{r\in F(s,k|S)}\epsilon_{s,r}] \\
            & \leq \frac{1}{k}|\sum_{r\in \hat{F}_d(s,k|f,S)} \epsilon_{s,r}| + \frac{1}{k}|\sum_{r\in F(s,k|S)}\epsilon_{s,r}|,
    \end{aligned}
\end{equation}

\noindent we can apply concentration inequalities to obtain following probabilistic bound

\begin{theorem}\label{theorem:top-k-prob}
	Assume for any given $s$, $\epsilon_{s,r}$ are independently and identically distributed zero-mean random variables on $R$. If we further assume
    \begin{equation}\nonumber
        E\epsilon_{s,r}^2 \leq \varepsilon, |\epsilon_{s,r}| \leq M \ a.s., \ \forall r \in S,
    \end{equation}
    we have
    \begin{equation}\nonumber
        Pr(\delta(s,k,f|S) \geq 2c\varepsilon) \leq 4\exp(\frac{-k c^2 \varepsilon}{2 + \frac{2}{3}Mc}), \ \forall c \geq 1.
    \end{equation}
\end{theorem}
\begin{proof}
We can then apply the Bernstein inequality on each part of the right hand side of \eqref{delta_form}:
\begin{equation}\nonumber
    \begin{aligned}
        Pr(\delta(s,k,f|S) \geq 2c\varepsilon) \leq \ & Pr(|\sum_{r\in \hat{F}_d(s,k|f,S)} \epsilon_{s,r}| \geq kc\varepsilon) \\
        & + Pr(|\sum_{r\in F(s,k|S)}\epsilon_{s,r}| \geq kc\varepsilon) \\
        \leq \ & 4 \exp(\frac{-(k c \varepsilon)^2}{2k\varepsilon + \frac{2}{3}Mkc\varepsilon}) \\
        \leq \ & 4 \exp(\frac{-k c^2 \varepsilon}{2 + \frac{2}{3}Mc}).
    \end{aligned}
\end{equation}
\end{proof}

Such probabilistic bound is less ideal than the deterministic bound in Theorem \ref{theorem:top-k}. It is more effective for accuracy of threshold-based queries since the corresponding $k$ will be much greater than in the top-$k$ queries and the bound will be much tighter.

\section{Design of Embedding Structure}
\label{sec:design}
In last section, we show that if we have an accurate embedding function $f$, the approximation of our correlation search will also be accurate. Nevertheless, it is not easy to find such an $f$. In this section we will discuss several potential neural network structures that we considered.

\subsection{Embedding on the time-domain}

For direct embedding of time series, a natural selection is the Recurrent neural network (RNN). RNN is a suitable model for sequential data, and recently it has been widely applied in sequence-related tasks such as machine translation \cite{Bahdanau_2014} and video processing \cite{Donahue_2015}. Specifically, given a sequence $[X_1,\cdots,X_M]$, a (single-layer) RNN tries to encode first $m$ observations $[X_1,\cdots,X_m]$ into state $h_m$ with the following dynamic processes
\begin{equation}
	h_{m+1} = g(X_{m+1},h_m)
\end{equation}

\noindent where $g$ is the same function across $m$ and is called a ``unit''. Common choices of RNN units include Long-short term memory (LSTM) \cite{Hochreiter_1997} and Gate recurrent unit (GRU) \cite{Cho_2014}. We can then take the last state $h_M$ as the final embedding vector
\begin{equation}
	f(s) = h_M.
\end{equation}
This approach has a disadvantage that the training algorithm is very slow when the time series is long because of the recurrent computations. Moreover, in the experiments we observed that this method does not give good results for long time series.

\subsection{Embedding on the frequency-domain}

To avoid the time-consuming and ineffective recurrent computations, we first use DFT to transform a time series $s$ into a frequency domain $\mathbf{s} = DFT(s)$. Next, we use a multi-layer fully connected (dense) network with ReLU activation functions
\begin{equation}\nonumber
\text{ReLU}(x) = \max\{x,0\}
\end{equation}

\noindent to transform the frequency vector $\mathbf{s}$ into an embedded vector with a desirable size. At last, we applied $l_2$ normalization to the embedding vector to project it on a unit sphere. This normalization step realign the scale of the Euclidean distances between embedding vectors with the correlation we aim to approximate.

\section{Training embedding neural networks}
\label{sec:train}
In this section we describe how we train the embedding neural networks for different objective functions.
\subsection{Approximation of correlation}
For Problem \ref{prob:approximation}, we sample a random batch of pair of time series $(s, r)$ and minimize the following loss function at each training iteration:
\begin{equation}\nonumber
   L_{approximate} = |\|f(\mathbf{s}|\theta) - f(\mathbf{r}|\theta)\|_2^2 - 2(1 - corr(s,r))|.
\end{equation}

% It is important to notice that there is a constant factor of 2 multiplied to the Euclidean distance between two embedding vectors. As have been discussed, because DFT is a symmetric vector,  the embedding method  using DFT only considers at most half of the number of coefficients in the frequency domain.
We call this method \textit{Chronos\footnote{Chronos in Greek means time.} Approximation} algorithm in the experiments.

\subsection{Top-$k$ correlation search}
For Problem \ref{prob:top-k}, we sample a random batch of tuple of time series $(s, r, u)$ and minimize the following loss function at each training iteration:
\begin{equation}\nonumber
    \begin{aligned}
        L_{order} = & |(\|f(\mathbf{r}|\theta)-f(\mathbf{s}|\theta)\|_2^2 - \|f(\mathbf{r}|\theta)-f(\mathbf{u}|\theta)\|_2^2) \\
        & - 2(corr(r,u) - corr(r,s))|
    \end{aligned}
\end{equation}
We call this method \textit{Chronos Order} algorithm in the experiments because it tries to approximate the correlation order. According to Corrolary \ref{cor:top-k}, if we optimize $L_{order}$ we also approximate the solutions for Problem \ref{prob:top-k}.
\subsection{Symmetry correction}
Recall that when  $s$ is a real number time series, its DFT $\mathbf{s}$ is symmetric, i.e. $\mathbf{s}_i$ is a complex conjugate of $\mathbf{s}_{M-i}$. Because of this reason, if $M=2*m$, we should have $d_M(\mathbf{s}, \mathbf{r}) = 2*d_m(\mathbf{s}, \mathbf{r}) = 2 - 2corr(s,r)$. Thanks to the symmetry property, people only need to use $m < \frac{M}{2}$ coefficients to approximate the correlation. In such case the approximation of  $2 - 2corr(s,r)$ is corrected to $2*d_m(\mathbf{s}, \mathbf{r})$.

Therefore, in our implementation of the Chronos  algorithms we use a corrected version of the loss functions as follows to take into account the symmetry of the embedding:
\begin{equation}\nonumber
   L_{approximate} = |2\|f(\mathbf{s}|\theta) - f(\mathbf{r}|\theta)\|_2^2 - 2(1 - corr(s,r))|.
\end{equation}
\begin{equation}\nonumber
    \begin{aligned}
        L_{order} = & |2(\|f(\mathbf{r}|\theta)-f(\mathbf{s}|\theta)\|_2^2 - \|f(\mathbf{r}|\theta)-f(\mathbf{u}|\theta)\|_2^2) \\
        & - 2(corr(r,u) - corr(r,s))|
    \end{aligned}
\end{equation}

\subsection{Datasets}
We used four real-world datasets to validate our approaches. One of these datasets is the daily stock indices crawled from the Yahoo finance \cite{YahooFin} using public python API. The other datasets are chosen from the UCR time series classification datasets \cite{UCRArchive}. These datasets were chosen because each time series is long (at least 400 readings) and the number of time series is large enough (at least a few thousands). Overall the characteristics of these datasets are described in Table \ref{tab:data size}. We preprocessed the data by adding dummy timestamps to the UCR time series classification datasets since our scripts do not assume time series sampled with the same resolution. The size of the data reported in Table \ref{tab:data size} is the disk size calculated based on the du command in Linux.

\subsection{Experiment settings}
For each dataset, we randomly split the data into training, validation and test with the ratios 80-10-10\%. We use the training set to train the embedding neural networks and validate it on the validation set. The test set is used to estimate the query accuracy. In particular, for the top-$k$ highest correlation queries, each time series in the test set is considered as a target series, and all the other series in the training set is considered as the database from which we look for correlated time series to the target.  For correlation approximation queries, we randomly permute the test series and create pairs of time series from the permutation. Since the test set is completely independent from training and validation, the reported results are the proper estimate of the query accuracy in practice.
\begin{table}
  \begin{tabular}{ | c | c | c | c |}
    \hline
    \textbf{Data} & \textbf{\# time series} & \textbf{length} & \textbf{size} \\ \hline
    Stock & 19420 & 1000 & 400  MB \\ \hline
    Yoga &  3300 &  426 & 26 MB   \\ \hline
    UWaveGesture &  4478 & 945  & 71 MB   \\ \hline
    StarLightCurves &  9236 &  1024 & 181 MB   \\ \hline
  \end{tabular}
  \caption{Datasets used in experiments} \label{tab:data size}
\end{table}

We fix parameters of the neural network to the default values described in Table \ref{tab:parameters}. Although we didn't fine-tune the network hyper-parameters and explore deeper neural network structures, the obtained results are already significantly better than the baseline algorithms. In practice, automatic network search and tuning can be done via Bayesian optimization \cite{Snoek_2012}. Since search for the optimal network structure and optimal network hyper-parameters is out of the scope of this work, we leave this problem as future work.

\begin{table}[h]
  \small
  \begin{center}
  \caption {Parameter settings} \label{tab:parameters}
  \begin{tabular}{ | c | c | }

    \hline
    \textbf{parameter} & \textbf{value} \\ \hline

    optimization algorithm & ADAM \cite{Kingma_2014} \\ \hline
    learning rate & 0.01 \\ \hline
    weight initialization & Xavier \cite{Glorot_2010} \\ \hline
    number of hidden layers & 1 \\ \hline
    hidden layer size & 1024 \\ \hline
	mini-batch size &  256 \\ \hline
	training iterations & 10000 mini-batches \\ \hline

  \end{tabular}
  \end{center}
\end{table}

The network was implemented in TensorFlow and run in a Linux system with two Intel cores, one NVDIA Tesla K40 GPU. All the running time is reported in the given system.

\begin{figure*}
    \begin{center}
        \includegraphics[width = 1.0\textwidth]{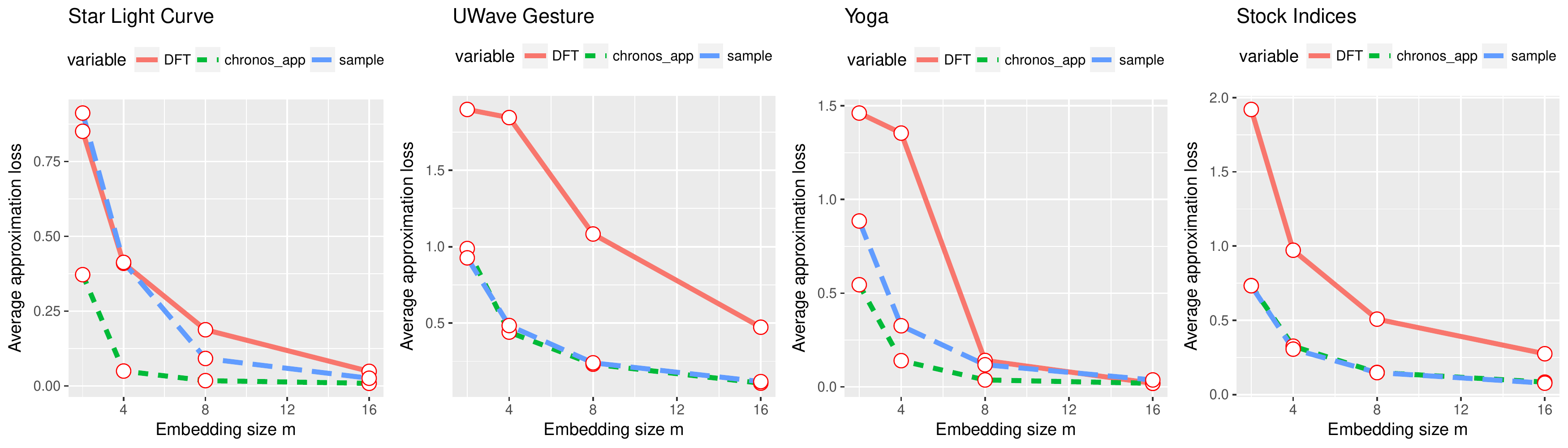}
    \caption{Results on correlation approximation}
    \label{fig:approximate}
    \end{center}
\end{figure*}

\subsection{Results on correlation approximation}
Figure \ref{fig:approximate} shows the comparison between the \textit{Chronos approximation} and DFT for approximated solutions to Problem \ref{prob:approximation}. Each subplot in the figure corresponds to a dataset. The vertical axis shows the approximation loss in the test set, while the horizontal axis shows the embedding size $m$, varying from 2 to 16. As can be observed, our method outperformed DFT in all the test cases with a significant margin. In particular, the approximation error was reduced at least by a half, especially for small embedding size. These empirical results confirm our theoretical analysis on the sub-optimality of DFT for Problem \ref{prob:approximation}.

It is interesting that the down-sample approach, although very simple, works very well in UWave and Stock datasets. Both Chronos and down-sample outperformed DFT in all cases. One possible explanation is that DFT seems to preserve approximation accuracy for high correlation (see Lemma 2 in \cite{Zhu_2002})  rather than optimize correlation approximation for a randomly picked pair of time series.

\subsection{Results on top-$k$ correlation search}
Figure \ref{fig:approximate} shows the comparison between the \textit{Chronos order}, the \textit{Chronos approximation}, DFT and down-sample for approximated solutions to Problem \ref{prob:top-k}. Each subplot in the figure corresponds to a pair of dataset and query size $k$ (varied from 10 to 100). The vertical axis shows the precision at $k$ in the test set, while the horizontal axis shows the embedding size $m$, varying from 2 to 16.

Although Theorem \ref{theorem:top-k} and Corollary \ref{cor:top-k} provide the same theoretically guarantee for the approximation bound of \textit{Chronos order} and \textit{Chronos order} on Problem \ref{prob:top-k}, our results show that the \textit{Chronos order} outperforms the \textit{Chronos approximation} in terms of the precision. These algorithms differ only in the loss function in training the networks: while \textit{Chronos approximation} directly minimizes the loss of correlation approximation, the \textit{Chronos order} also tries to preserve order information. This result implies that using the right loss function for a given query type is crucial in achieving good results.

In all cases, the \textit{Chronos order} algorithm outperforms the DFT baseline algorithms. On the other hand, except for the Stock dataset, \textit{Chronos approximation} algorithm outperforms the DFT baseline algorithm. The overall improvement varies from 5\% to 20\% depending on the value of $k$ and $m$. The down-sample algorithm although works very well in the Yoga dataset, its accuracy is very low for the other cases.

\begin{figure*}
    \begin{center}
        \includegraphics[width = 1.0\textwidth]{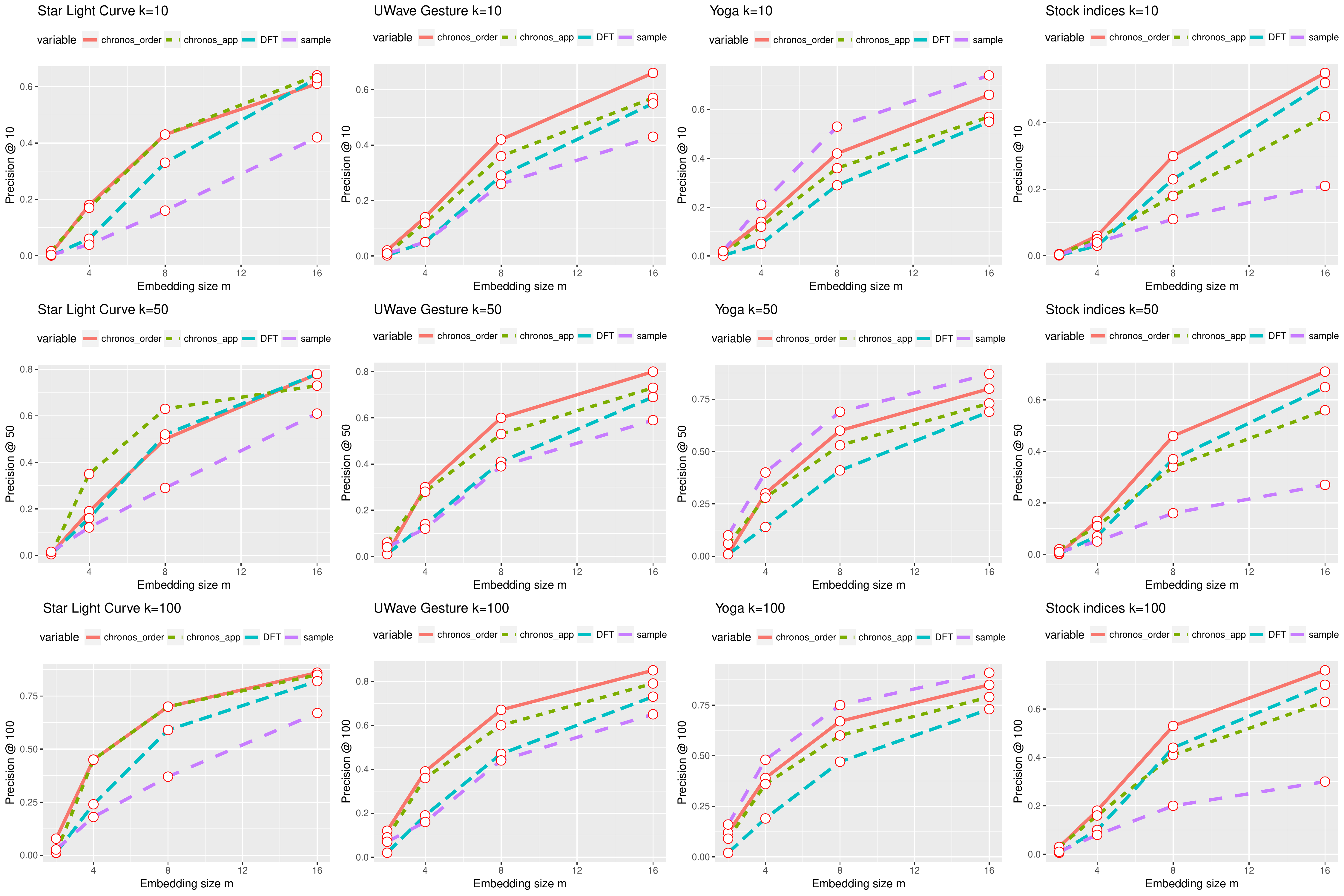}
    \caption{Results on top-$k$ correlation search}
    \label{fig:topk}
    \end{center}
\end{figure*}

\subsection{Training \& Query time}
Another important evaluation metric on our methods' performance is the time efficiency in training the embedding network $f(\mathbf{s}|\theta)$ and in subsequent function evaluation for incoming queries. Table \ref{tab:training} shows the training time of the \textit{Chronos Order} method in four datasets when $k=100, m=16$. For cases with other values of $k$ and $m$, the results are similar. We can see that it only takes about 15 minutes to train the network with 10000 mini-batches.
\begin{table}
  \begin{tabular}{ | c | c |}
    \hline
    \textbf{Data} & \textbf{\textit{Chronos order}}  \\ \hline
    Stock & 947  \\ \hline
    Yoga &  914   \\ \hline
    UWaveGesture &  894    \\ \hline
    StarLightCurves &  950  \\ \hline
  \end{tabular}
  \caption{Training time (seconds) with 10000 mini-batches} \label{tab:training}
\end{table}

Table \ref{tab:queries} shows the query time of different algorithms when $k=100, m=16$. Again, for cases with other values of $k$ and $m$, the results are similar. As can be seen, the queries time is not significantly influenced by the additional tasks for evaluating the function $f(\mathbf{s}|\theta)$. Since the queries time is dominated by the k-d tree traversal, time for evaluating the function $f(\mathbf{s}|\theta)$ is negligible. Query time using DFT is slightly faster than the the query time of Chronos order. Both methods have very quick response time per query being less than one millisecond in all experiments.

\begin{table}
  \begin{tabular}{ | c | c | c |}
    \hline
    \textbf{Data} & \textbf{\textit{Chronos order}} & \textbf{DFT}  \\ \hline
    Stock & 0.9 & 0.8 \\ \hline
    Yoga &  0.7 &  0.3    \\ \hline
    UWaveGesture &  0.7 & 0.3    \\ \hline
    StarLightCurves &  0.6 &  0.4   \\ \hline
  \end{tabular}
  \caption{Average top-$k$ query time (milliseconds)} \label{tab:queries}
\end{table}

\section{Related works}
\label{sec:related}
\textbf{Correlation related queries on time series.}
During the past several decades, people investigated a massive amount of correlation-related query problems on time series datasets in the literature. As our review will be far from exhaustive, here we rather provide several problem categories and some corresponding examples.

One major classification on these problems is whether a specific time series is given in the query. For example, some works \cite{Zhu_2002,Sakurai_2005,Mueen_2010,Li_2013} are interested in computing correlation among all sequence pairs or providing all one that are correlated (over a certain threshold), while others \cite{Faloutsos_1994,Rak_2012} aims at finding most correlated sequences in the dataset for the given one. There are also some papers addressing both sides \cite{Agrawal_1993}. In this paper, we focus on the later category that a specific time series is assumed to be given in the query.

Another major distinction is whether the object they consider is the whole sequence or a subsequence. For problems focusing on whole sequence \cite{Agrawal_1993,Zhu_2002,Mueen_2010,Rak_2012}, not only the query time series but also all sequences from the dataset are assume to have (almost) the same length. Our problems also make this assumption. For problems concerning subsequences \cite{Faloutsos_1994}, the query time series is generally much shorter than those in the dataset, which are also assumed to have arbitrary lengths.

Finally, an important distinction is whether the data is provided offline or collected online. On one hand, problems in the online setting \cite{Zhu_2002,Sakurai_2005} typically address the varying nature in data streams and discover stationary submodule in the computation procedure to boost query efficiency. On the other hand, problems in the offline settings \cite{Mueen_2010,Rak_2012}, including ours, usually deal with the difficulties brought by the massive data size.

\textbf{Dimension reduction.}
As far as we know, we are the first to suggest using function approximation, rather than algorithms for feature extraction, for dimension reduction in time series indexing and correlation searching problems. The application of dimension reduction techniques to improve time series correlation search efficiency is first introduced by Agrawal \textit{et al.} \cite{Agrawal_1993}. DFT and discrete wavelet transform (DWT) are the two most commonly used methods. The usage of DFT was introduced by Agrawal \textit{et al.} \cite{Agrawal_1993}, who also constructed $R^*$-tree for indexing on the transformed Fourier coefficients. Chan and Fu \cite{Chan_1999}, on the other hand, were the first to use wavelet analysis for time series embedding. Regardless of the significant difference between DFT and DWT approaches, it is shown that they have similar performance on query precision in a later paper by Wu \textit{et al.} \cite{Wu_2000}. Therefore, in this paper we only consider DFT as our baseline method.

We also notice that since Agrawal \textit{et al.} \cite{Agrawal_1993} researchers suggested many different dimensionality reduction techniques to provide tighter approximation bound for pruning. To name a few, there are Adaptive Piecewise Constant Approximation (APCA) \cite{Keogh_2001_APCA}, Piecewise Aggregate Approximation (PAA) \cite{Keogh_2001_PAA}, multi-scale segment mean (MSM) \cite{Lian_2009}, Chebyshev Polynomials (CP) \cite{Cai_2004}, and Piecewise Linear Approximation (PLA) \cite{Chen_2007}. However, according to the comparison work by Ding \textit{et al.} \cite{Ding_2008}, none of these significantly outperform DFT and DWT.

\textbf{Similarity measures.}
Though simple, Pearson correlation and Euclidean distance ($l_2$-norm) are still the most commonly used similarity measure in the literature. We also focus on this measure in this paper. However, Euclidean distance suffers from drawbacks such as fixed time series length and its applications in general are much restricted. Therefore, in many applications, more flexible similarity measure such as Dynamic time warping (DTW) \cite{Sakoe_1978} is more appropriate. DTW is an algorithm for similarity computation between general time series and can deal with time series with varying lengths and local time shifting. One of the major drawbacks of DTW is its computation complexity; therefore, most related research on DTW develop and discuss its scalable application in massive datasets \cite{Berndt_1994,Yi_1998,Keogh_2005,Rak_2012}.

There are also other similarity measures such as Longest Common Subsequence (LCSS) \cite{Boreczky_1996,Vlachos_2002}, but rather than concerning generalizability, they either focus on a specific type of data or problem related to the $l_2$-norm and DTW. They therefore did not attract much attention in the past decade. Interested readers are again referred to the comparison work by \cite{Ding_2008} for more details on similarity measures.

\textbf{Neural network approximation.}
With the recent successes in deep learning, neural networks again become the top choice for function approximators. Recent work by Kraska \textit{et al.} \cite{Kraska_2017} further illustrates that neural networks, with proper training, can beat baseline database indexes in terms of speed and space requirement.

There are also more neural-network-based methods in time series related problems. Cui \textit{et al.} \cite{Cui_2016} design an multi-scale convolution neural network for time series classification. Zheng \textit{et al.} \cite{Zheng_2015} apply convolution neural network for distance metric learning and classification. Both works show that the designed structures can achieve baseline performance when data is sufficiently large. However, in these works the networks are supervised to learn the mapping from time series to the given labels, while in this paper the network is developed for unsupervised dimension reduction and feature extraction of dataset structure.

\textbf{Approximation and exact algorithm}
When the data is small which fits memory entirely, the exact algorithm such as MASS by Mueen et al. \cite{mass} can be used to efficiently retrieve the exact correlations. However, in applications like the ones running at edge  when the data size is so big that does not fit available memory, we need a small index of the data which is several orders of magnitude smaller than the original size. One possible way to resolve the resource limitation issue is to down-sample the data to a desirable size that fits memory. This approach is similar to the embedding approach proposed in this paper  which trades between accuracy and resource usage. However, as we will see in the experiments, down-sampling does not work well because it is not designed to optimize the objective of the interested problems.

\textbf{Metric learning for time-series} Recently, there is a trend to use deep neural networks for learning metric for different time-series mining tasks. For instance, Garreau et al. proposed to learn a Mahalanobis distance to perform alignment of multivariate time series \cite{alignment}, Zhengping et al. \cite{decade}  proposed a method for learning similarity between multi-variate time-series. Pei et al. \cite{wenjie} used recurrent neural networks to learn time series metric for classification problems. Do et al. \cite{metricsknn} proposed a metric learning approach for k-NN classification problems. Our method is tightly related to these works because it can be considered as a metric learning approach. However, the key difference between our work and the prior work is in the objective function we optimize. While the objectives in the prior arts are designed for classification or time-series alignment, our objectives are designed to approximate correlation between time series in similarity search problems. Moreover, our significant contribution lies in the theoretical analysis of the approximation bounds for each specific type of the similar search queries.    
\section{Experiments}
\label{sec:experiments}
In this section, we discuss the experiment results. The  DFT method is considered as the baseline and compared with \textit{Chronos approximation} and \textit{Chronos order} methods in several different types of queries. Besides, we also compared the proposed approaches to the down-sample method in which the time series were embedded to a small embedding space by downs-sampling.   We first introduce the datasets and experimental settings.

\section{Conclusions and future work}
\label{sec:final}
In this paper, we propose a general approximation framework for a variety of correlation search queries in time series datasets. In this framework, we use Fourier transform and neural network to embed time series into low-dimensional Euclidean space, and construct nearest neighbor search structures in the embedding space for time series indexing to speed up queries. Our theoretical analysis illustrates that our method's accuracy can be guaranteed under certain regularity conditions, and our experiment results on real datasets indicate the superiority of our method over the baseline solution.

Several observations in this work are interesting and require more attention. First, the approximation accuracy of our method varies significantly across datasets. Therefore, it is crucial in later real-world applications to more systematically evaluate the applicability of our method and design network structure for specific datasets. Second, the selection of distance function $d$ can be important in improving the performance of our embedding method, since it might be able to balance between the approximation capability of the neural network and the internal similarity within the datasets.

Our work can be further extended in several directions. First, our method can only deal with time series with equal length. However, a large number of real-world time series are sampled with irregular time frequency or during arbitrary time period. Therefore, embedding methods that can deal with general time series should be explored in the future. Second, currently we only consider Pearson correlation as the evaluation metric for similarity. We can investigate the applicability of our framework on other types of similarity measure such as the dynamic time wrapping (DTW) measure. Finally, in some of the aforementioned applications such as time series forecasting, correlation search for the target time series might not provide the most essential information in improving the ultimate objectives. End-to-end systems explicitly connecting correlation search techniques and application needs might therefore be more straightforward and effective.

% \section{Acknowledgements}
% We would like to thank Dr. Tran Ngoc Minh and Dr. Martin Wistuba for useful discussion and suggestion during the development of the project. This research was done while the first author was a research intern student at IBM research.

\bibliographystyle{acm}
\bibliography{ts}

\end{document}